\documentclass{article}

\usepackage{times}
\usepackage{graphicx} 
\usepackage{subfigure} 

\usepackage{natbib}

\usepackage{algorithm}
\usepackage{algorithmic}
\usepackage{amsfonts}
\usepackage{amsmath}
\usepackage[titletoc]{appendix}
\usepackage{bm}
\usepackage{dsfont} 
\usepackage{multicol}
\usepackage{hyperref}
\usepackage{lipsum}
\usepackage{bbm}
\usepackage{mathtools}
\usepackage{epstopdf}
\usepackage{amssymb} 
\usepackage{multirow}

\allowdisplaybreaks

\newtheorem{theorem}{Theorem}[section]

\newtheorem{proposition}[theorem]{Proposition}

\newenvironment{proof}[1][Proof]{\begin{trivlist}
\item[\hspace \labelsep {\bfseries #1}]}{\end{trivlist}}

\newcommand{\qed}{\nobreak \ifvmode \relax \else
      \ifdim\lastskip<1.5em \hskip-\lastskip
      \hskip1.5em plus0em minus0.5em \fi \nobreak
      \vrule height0.75em width0.5em depth0.25em\fi}

\usepackage{lineno}
\usepackage[accepted]{icml2015}
\icmltitlerunning{Scalable Nonparametric Bayesian Inference on Point  Processes with Gaussian Processes}

\begin{document} 

\twocolumn[
\icmltitle{Scalable Nonparametric Bayesian Inference on Point  Processes with Gaussian Processes}

\icmlauthor{Yves-Laurent Kom Samo}{yves-laurent.komsamo@eng.ox.ac.uk}
\icmlauthor{Stephen Roberts}{sjrob@robots.ox.ac.uk}
\icmladdress{Deparment of Engineering Science and Oxford-Man Institute, University of Oxford}

\icmlkeywords{Point Processes, Cox Processes, Gaussian Processes }
\vskip 0.3in
]

\begin{abstract}
In this paper we propose an efficient, scalable non-parametric
Gaussian Process model for inference on Poisson
Point Processes. Our model does not resort to gridding the domain or to
introducing latent \textit{thinning} points. Unlike competing models
that scale as $\mathcal{O}(n^3)$ over $n$ data points, our model has a complexity
$\mathcal{O}(nk^2)$ where $k\ll n$. We propose a MCMC sampler and show
that the model obtained is faster, more
accurate and generates less correlated samples than competing approaches
on both synthetic and real-life data. Finally, we show that our model
easily handles data sizes not considered thus far by alternate
approaches. 
\end{abstract}

\section{INTRODUCTION}
Point processes are a standard model when the objects of study are the number and repartition of otherwise identical points on a domain, usually time or space. The Poisson Point Process is probably the most commonly used point process. It is fully characterised by an intensity function that is inferred from the data. Gaussian Processes have been successfully used to form a prior over the (log-) intensity function for applications such as astronomy \cite{Gregory92}, forestry \cite{heik99}, finance \cite{basu02}, and neuroscience \cite{Cunni08b}. We offer extensions to existing work as follows:  we develop an exact non-parametric Bayesian model that enables inference on Poisson processes. Our method scales linearly with the number of data points and does not resort to gridding the domain. We derive a MCMC sampler for core components of the model and show that our approach offers a faster and more accurate solution, as well as producing less correlated samples, compared to other approaches on both real-life and synthetic data.

\section{RELATED WORK}
Non-parametric inference on point processes has been extensively studied in the literature. \cite{Rathbum94} and \cite{Moller98} used a finite-dimensional piecewise constant log-Gaussian for the intensity function. Such approximations are limited in that the choice of the grid on which to represent the intensity function is arbitrary and one has to trade-off precision with computational complexity and numerical accuracy, with the complexity being cubic in the precision and exponential in the dimension of the input space.  \cite{Kottas06, Kottas07} used a Dirichlet process mixture of Beta distributions as prior for the normalised intensity function of a Poisson process. \cite{Cunni08} proposed a model using Gaussian Processes evaluated on a fixed grid for the estimation of intensity functions of renewal processes with log-concave renewal distributions. They turned hyper-parameters inference into an iterative series of convex optimization problems, where ordinarily cubic complexity operations such as Cholesky decompositions are evaluated in $\mathcal{O}(n\log{n})$ leveraging the uniformity of the grid and the log-concavity of the renewal distribution. \cite{Murray09} proposed an exact Markov Chain Monte Carlo (MCMC) inference scheme for the posterior intensity function of a Poisson Process with a Sigmoid Gaussian prior intensity, or equivalently a Cox Process \cite{Cox55} with Sigmoid Gaussian stochastic intensity. The authors simplified the likelihood of a Cox process by introducing latent \textit{thinning points}. The proposed scheme has a complexity exponential in the dimension of the input space, cubic in the number of data and thinning points, and performs particularly poorly when the data are sparse. \cite{gunter} extended this model to structured point processes. \cite{YWT11} used \textit{uniformization} to produce exact samples from a non-stationary renewal process whose hazard function is modulated by a Gaussian Process, and consequently proposed an MCMC sampler to sample from the posterior intensity of a unidimensional point process. Although the authors have illustrated that their model is faster than \cite{Murray09} on some synthetic and real-life data, their method still scales cubically in the number of thinned and data points, and is not applicable to data in dimension higher than 1, such as spatial point processes.

\section{MODEL}
\subsection{Setup}
We are tasked with making non-parametric Bayesian inference on the intensity function of a Poisson Point Process assumed to have generated a dataset $\mathcal{D} =\{ s_1, ..., s_n\}$. To simplify the discourse without loss of generality, we will assume that data points take values in $\mathbb{R}^d$.

Firstly, let us recall that a \textit{\textbf{Poisson Point Process}} (PPP) on a bounded domain $\mathcal{S} \subset \mathbb{R}^d$ with non-negative \textit{\textbf{intensity function}} $\lambda$ is a locally finite random collection of points in $\mathcal{S}$ such that the numbers of points occurring in disjoint parts $B_i$ of $\mathcal{S}$ are independent and each follows a Poisson distribution with mean $\int_{B_i} \lambda(s) ds$.

The likelihood of a PPP is given by:
\begin{align}
\label{eq:likelihood_poisson}
L(\lambda|s_1, ...,s_n) = \exp\left(-\int_{\mathcal{S}} \lambda(s) ds\right) \prod_{i=1}^n \lambda(s_i)
\end{align}

\subsection{Tractability discussion}
The approach adopted thus far in the literature to make non-parametric Bayesian inference on Point Process using Gaussian Processes (GP) \cite{Rassmussen06} consists of putting a \textit{functional prior} on the intensity function in the form of a positive function of a GP: $\lambda(s)= f(g(s))$ where $g$ is drawn from a GP and $f$ is a positive function. Examples of such $f$ include the exponential function and a scaled sigmoid function \cite{Murray09, YWT11}. This approach can be seen as a Cox Process where the stochastic intensity follows the same dynamics as the functional prior. When the Gaussian Process used has almost surely continuous paths, the random vector
\begin{align}
\label{eq:tractable_joint}
(\lambda(s_1),..., \lambda(s_n), \int_{\mathcal{S}} \lambda(s) ds)
\end{align}
provably admits a probability density function (pdf). Moreover, we note that any piece of information not contained in the implied pdf over the vector in Equation (\ref{eq:tractable_joint}) will be lost as the likelihood only depends on those variables. Hence, given a functional prior postulated on the intensity function, the only necessary piece of information to be able to make a full Bayesian treatment is the implied joint pdf over the vector in Equation (\ref{eq:tractable_joint}). 

For many useful transformations $f$ and covariance structures for the GP, the aforementioned implied pdf might not be available analytically. We note however that there is no need to put a \textit{functional prior} on the intensity function. In fact, for every \textit{finite-dimensional prior} over the vector in Equation (\ref{eq:tractable_joint}), there exists a Cox process with an a.s. $\mathcal{C}^\infty$ intensity process that coincides with the postulated prior (see appendix for the proof).

This approach is similar to that of \cite{Kottas06}. The author regarded $I =\int_{\mathcal{S}} \lambda(s) ds$ as a random variable and noted that $p(s)=\frac{\lambda(s)}{\int_{\mathcal{S}} \lambda(s) ds}$ can be regarded as a pdf whose support is the domain $\mathcal{S}$. He then made inference on $(I, p(s_1), ..., p(s_n))$, postulating as prior that $I$ and $(p(s_1), ..., p(s_n))$ are independent, $I$ has a Jeffreys prior and $(s_1, ..., s_n)$ are i.i.d. draws from a Dirichlet Process mixture of Beta with pdf $p$. 

The model we present in the following section puts an appropriate \textit{finite-dimensional} prior on $(\lambda(s_1),..., \lambda(s_n), \lambda(s^\prime_1),..., \lambda(s^\prime_k), \int_{\mathcal{S}} \lambda(s) ds)$ for some inducing points $s^\prime_j$ rather than putting a \textit{functional prior} on the intensity function directly.

\subsection{Our model}
\subsubsection{Intuition}
The intuition behind our model is that the data are not a `natural grid' at which to infer the value of the intensity function. For instance, if the data consists of 200,000 points on the interval $[0, 24]$ as in one of our experiments, it might not be necessary to infer the value of a function at 200,000 points to characterise it on $[0, 24]$. Instead, we find a small set of inducing points $\mathcal{D}^\prime = \{s^\prime_1, ..., s^\prime_k \}, k \ll n $ on our domain, through which we will define the prior over the vector in Equation (\ref{eq:tractable_joint}) augmented with $\lambda(s^\prime_1),..., \lambda(s^\prime_k)$. The set of inducing points will be chosen so that knowing $\lambda(s^\prime_1),..., \lambda(s^\prime_k)$ would result in knowing the values of the intensity function elsewhere on the domain, in particular $\lambda(s_1),..., \lambda(s_n)$, with `arbitrary certainty'. We will then analytically integrate out the dependency in $\lambda(s_1),..., \lambda(s_n)$ from the posterior, thereby reducing the complexity from cubic to linear in the number of data points without `loss of information', and reformulating our problem as that of making exact Bayesian inference on the value of the intensity function at the inducing points. We will then describe how to obtain predictive mean and variance of the intensity function elsewhere on the domain from training.

\subsubsection{Model specification}
Let us denote by $\lambda^{*}$ a positive stochastic process on $\mathcal{S}$ such that $\log{\lambda^{*}}$ is a stationary Gaussian Process  with covariance kernel $\gamma^{*}: (s_1, s_2) \to \gamma^{*}(s_1, s_2)$ and constant mean $m^{*}$. Let us further denote by $\hat{\lambda}$ a positive stochastic process on $\mathcal{S}$ such that $\log{\hat{\lambda}}$ is a  \textit{\textbf{Conditional Gaussian Process}} \textit{coinciding with $\log{\lambda^{*}}$ at $k$ inducing points $\mathcal{D}^\prime = \{s^\prime_1, ..., s^\prime_k \}, k \ll n $}. That is, $\log{\hat{\lambda}}$ is the non-stationary Gaussian Process whose mean function $m$ is defined by
\begin{align}
\label{eq:cond_mean}
m(s) = m^{*} + \Sigma^{*}_{s\mathcal{D}^\prime}\Sigma^{*-1}_{\mathcal{D}^\prime\mathcal{D}^\prime}G
\end{align}
where $G=\big(\log{\lambda^{*}(s^\prime_1)} - m^*, ..., \log{\lambda^{*}(s^\prime_k)} - m^*\big)$ and $\Sigma^{*}_{XY}$ is the covariance matrix between the vectors $X$ and $Y$ under the covariance kernel $\gamma^*$. Moreover,  $\log{\hat{\lambda}}$ is such that for every vector $S_1$ of points in $\mathcal{S}$, the auto-covariance matrix $ \Sigma_{S_1S_1}$ of the values of process at $S_1$  reads\footnote{The positive definitiveness of the induced covariance kernel $\gamma$ is a direct consequence of the positive definitiveness of $\gamma^{*}$.}
\begin{align}
\label{eq:cond_cov}
& \Sigma_{S_1S_1} = \Sigma^{*}_{S_1S_1} - \Sigma^{*}_{S_1\mathcal{D}^\prime}\Sigma^{*-1}_{\mathcal{D}^\prime\mathcal{D}^\prime}\Sigma^{*T}_{S_1\mathcal{D}^\prime}.
\end{align}

The prior distribution in our model is constructed as follows: 
\begin{enumerate}
  \item $\{\log{\lambda(s^\prime_i)}\}_{i=1}^{k}$ are samples from the stationary GP $\log{\lambda^{*}}$ at $\{s^\prime_i\}_{i=1}^{k}$ respectively, with $m^{*} =\log{\frac{\# \mathcal{D}}{\mu(\mathcal{S})}}$, where $\mu(\mathcal{S})$ is the size of the domain.
  \item $I=\int_{\mathcal{S}}\lambda(s) ds$ and  $\{\log{\lambda(s_j)}\}_{j=1}^{n}$ are conditionally independent given  $\{\log{\lambda(s^\prime_i)}\}_{i=1}^{k}$.
  \item Conditional on $\{\log{\lambda(s^\prime_i)}\}_{i=1}^{k}$, $\{\log{\lambda(s_j)}\}_{j=1}^{n}$ are independent, and for each $j \in [1..n]$ $\log{\lambda(s_j)}$ follows the same distribution as $\log{\hat{\lambda}(s_j)}$.
  \item Conditional on $\{\log{\lambda(s^\prime_i)}\}_{i=1}^{k}$, $I$ follows a Gamma distribution with shape $\alpha_I$ and scale $\beta_I$.
    \item The mean $\mu_I = \alpha_I \beta_I$ and variance $\sigma^2_I = \alpha_I \beta_I^2$ of $I$ are that of $\int_{\mathcal{S}}  \hat{\lambda}(s) ds$.
\end{enumerate}
Assertion 3. above is somewhat similar to the FITC model of \cite{FTCI}. 

This construction yields a prior pdf of the form:
\begin{align}
\label{eq:tractable_joint_prior_condgp}
& p\big(\log{\lambda(s_1)},..., \log{\lambda(s_n)}, \log{\lambda(s^\prime_1)}, ..., \log{\lambda(s^\prime_k)},  I,  \theta \big) \nonumber \\ &=   \mathcal{N}\big(\log{\lambda(s^\prime_1)}, ..., \log{\lambda(s^\prime_k)} \rvert m^{*} 1_{k}, \Sigma^{*}_{\mathcal{D}^\prime\mathcal{D}^\prime}\big)  \nonumber \\
& \times \mathcal{N}\big(\log{\lambda(s_1)},..., \log{\lambda(s_n)} \rvert  M, \text{diag}(\Sigma_{\mathcal{D}\mathcal{D}})\big)  \nonumber \\
& \times \gamma_d\big(I\rvert \alpha_I, \beta_I\big) \times p(\theta)
\end{align}

where $\mathcal{N}(. \rvert X, C)$ is the multivariate Gaussian pdf with mean $X$ and covariance matrix $C$, $M = (m(s_1), ..., m(s_n))$, $1_k$ is the vector with length k and elements 1, $\text{diag}(\Sigma_{\mathcal{D}\mathcal{D}})$ is the diagonal matrix whose diagonal is that of $\Sigma_{\mathcal{D}\mathcal{D}}$, $\gamma_d\big(x \rvert \alpha, \beta)$ is the pdf of the gamma distribution with shape $\alpha$ and scale $\beta$, and where $\theta$ denotes the hyper-parameters of the covariance kernel $\gamma^{*}$.

It follows from the fifth assertion in our prior specification that $\alpha_I = \frac{\mu_I^2}{\sigma_I^2} \text{ and } \beta_I = \frac{\sigma_I^2}{\mu_I}$. We also note that
\begin{align}
\label{eq:gamma_mean}
\mu_I &= \text{E}\big(\int_{\mathcal{S}}  \hat{\lambda}(s) ds\big)\nonumber \\
& = \int_{\mathcal{S}}  \text{E}\big(\exp(\log{\hat{\lambda}(s)}) \big)ds \nonumber \\
& = \int_{\mathcal{S}}  \exp(m(s)+\frac{1}{2} \gamma(s,s))ds
& := \int_{\mathcal{S}}  f(s)ds
\end{align}
and 
\begin{align}
\label{eq:gamma_variance}
\sigma_I^2 &= \text{E}\bigg(\big(\int_{\mathcal{S}}  \hat{\lambda}(s) ds\big)^2\bigg) -\mu_I^2\nonumber \\
&= \text{E}\bigg(\int_{\mathcal{S}}\int_{\mathcal{S}} \exp(\log{\hat{\lambda}(s_1)}+\log{\hat{\lambda}(s_2)}) ds_1 ds_2\bigg) -\mu_I^2\nonumber \\
&= \int_{\mathcal{S}}\int_{\mathcal{S}} \text{E}\bigg(\exp\big(\log{\hat{\lambda}(s_1)}+\log{\hat{\lambda}(s_2)}\big)\bigg) ds_1 ds_2 -\mu_I^2\nonumber \\
& = \int_{\mathcal{S}} \int_{\mathcal{S}}  \exp\big(m(s_1)+ m(s_2) + \gamma(s_1, s_2) + \frac{1}{2} \gamma(s_1,s_1) \nonumber \\
& + \frac{1}{2} \gamma(s_2,s_2)\big)ds_1 ds_2 -\mu_I^2 \nonumber \\
& := \int_{\mathcal{S}} \int_{\mathcal{S}} g(s_1, s_2) ds_1 ds_2 -\mu_I^2.
\end{align}

The integrals in Equations (\ref{eq:gamma_mean}) and (\ref{eq:gamma_variance}) can be easily evaluated with numerical methods such as Gauss-Legendre quadrature \cite{Hilde56}. 

In particular, when $\mathcal{S}=[a,b]$, 
\begin{align}
\mu_I \approx \frac{b-a}{2} \sum_{i=1}^{p} \omega_i f\left(\frac{b-a}{2} x_i + \frac{b+a}{2}\right) 
\end{align}
and 
\begin{align}
& \sigma_I^2 \approx \frac{(b-a)^2}{4} \sum_{i=1}^{p} \sum_{j=1}^{p} \omega_i \omega_j g\Big(\frac{b-a}{2} x_i + \frac{b+a}{2}, \frac{b-a}{2} x_j + \nonumber \\
& \frac{b+a}{2}\Bigg)  -\mu_I^2
\end{align}
where the roots $x_i$ of the Legendre polynomial of order $p$ and the weights $\omega_i$ are readily available from standard textbooks on numerical analysis such as \cite{Hilde56} and scientific programming packages (R, Matlab and Scipy). Extensions to rectangles in higher dimensions are straightforward. Moreover, the complexity of such approximations only depends on the number of inducing points and $p$ (see Equations (\ref{eq:cond_mean}) and (\ref{eq:cond_cov})), and hence \emph{scales well} with the data size.

A critical step in the derivation of our model is to analytically integrate out $\log{\lambda(s_1)},..., \log{\lambda(s_n)}$ in the posterior, to eliminate the cubic complexity in the number of data points.
To do so, we note that:
\begin{align}
\label{eq:moment_gen}
&\int_{(\mathbb{R}^d)^n}  \prod_{i=1}^{n} \lambda(s_i) \mathcal{N}\big(\log{\lambda(s_1)},..., \log{\lambda(s_n)}\rvert  M, \nonumber \\
& \text{diag}(\Sigma_{\mathcal{D}\mathcal{D}})\big)  d\log{\lambda(s_1)}...d\log{\lambda(s_n)} 
\nonumber \\
& = \text{E}\left(\exp\left(\sum_{i=1}^{n} \log{\lambda(s_i)}\right)\right) \nonumber \\
& = \exp(1_{n}^T M + \frac{1}{2} \text{Tr}( \Sigma_{\mathcal{D}\mathcal{D}}))
\end{align}
where the second equality results from the moment generating function of a multivariate Gaussian.

Thus, putting together the likelihood of Equation (\ref{eq:likelihood_poisson}) and Equation (\ref{eq:tractable_joint_prior_condgp}), and integrating out $\big(\log{\lambda(s_1)},..., \log{\lambda(s_n)}\big)$, we get:
\begin{align}
\label{eq:joint}
& p\big(\log{\lambda(s^\prime_1)}, ..., \log{\lambda(s^\prime_k)},  I, \theta \big \rvert \mathcal{D})  \nonumber \\ 
& \sim p(\theta)  \mathcal{N}\big(\log{\lambda(s^\prime_1)}, ..., \log{\lambda(s^\prime_k)} \rvert M^{*}, \Sigma^{*}_{\mathcal{D}^\prime\mathcal{D}^\prime}\big)  \nonumber \\
& \times \exp(1_{n}^T M + \frac{1}{2}  \text{Tr}( \Sigma_{\mathcal{D}\mathcal{D}}))  \exp(-I)  \gamma_d\big(I \rvert \alpha_I, \beta_I\big)
\end{align}

Finally, although our model allows for joint inference on the intensity function and its integral, we restrict our attention to making inference on the intensity function for brevity.  By integrating out  $I$ from Equation (\ref{eq:joint}), we get the new posterior:
\begin{align}
\label{eq:joint_int}
 p(\lambda, \theta \rvert \mathcal{D}) & :=  p\big(\log{\lambda(s^\prime_1)}, ..., \log{\lambda(s^\prime_k)},\theta\big \rvert \mathcal{D})  \\
&   \sim p(\theta)\exp(1_{n}^T M +  \frac{1}{2} \text{Tr}( \Sigma_{\mathcal{D}\mathcal{D}})) (1+\beta_I)^{-\alpha_I}  \nonumber \\
&  \times \mathcal{N}\big(\log{\lambda(s^\prime_1)}, ..., \log{\lambda(s^\prime_k)} \rvert M^{*}, \Sigma^{*}_{\mathcal{D}^\prime\mathcal{D}^\prime}\big) \nonumber
\end{align}
where we noted that the dependencies of Equation (\ref{eq:joint}) in $I$ is of the form $\exp(-x)\gamma_d(x \rvert \alpha, \beta)$ which can be integrated out as the moment generating function of the gamma distribution evaluated at $-1$, that is $ (1+\beta)^{-\alpha}$.

\subsubsection{Selection of inducing points}
Inferring the number $k$ and positions of the inducing points $s^\prime_i$ is critical to our model, as $k$ directly affects the complexity of our scheme and the positions of the inducing points affect the quality of our prediction. Too large a $k$ will lead to an unduly large complexity. Too small a $k$ will lead to loss of information (and subsequently excessively uncertain predictions from training), and might make assertion 2 of our prior specification inappropriate. For a given $k$, if the inducing points are not carefully chosen, the coverage of the domain will not be adapted to changes in the intensity function and as a result, the predictive variance in certain parts of the domain might considerably differ from the posterior variance we would have obtained, had we chosen inducing points in those parts of the domain.

Intuitively, a good algorithm to find inducing points should leverage prior knowledge about the smoothness, periodicity, amplitude  and length scale(s) of the intensity function to optimize for the quality of (post-training) predictions while minimising the number of inducing points. 

We use as utility function for the choice of inducing points:
\begin{align}
\label{eq:utility_inducing}
\mathcal{U}(\mathcal{D}^\prime) = \text{E}_{\theta}(\text{Tr}( \Sigma_{\mathcal{D}\mathcal{D}^{\prime}}^{*}(\theta) \Sigma_{\mathcal{D^\prime}\mathcal{D^\prime}}^{*-1}(\theta) \Sigma_{\mathcal{D}\mathcal{D}^{\prime}}^{*T}(\theta)))
\end{align}
where $\theta$ is the vector of hyper-parameters of the covariance kernel $\gamma^{*}$, and the expectation is taken with respect to the prior distribution over $\theta$. In other words, the utility of a set of inducing points is the expected total reduction of the (predictive) variances of $\log{\lambda(s_1)}, ...\log{\lambda(s_n)}$ resulting from knowing $\log{\lambda(s^\prime_1)}, ...,\log{\lambda(s^\prime_k)}$.

In practice, the expectation in Equation (\ref{eq:utility_inducing}) might not be available analytically. We can however use the Monte Carlo estimate:
\begin{align}
\label{eq:utility_inducing_MC}
\tilde{\mathcal{U}}(\mathcal{D}^\prime) &= \frac{1}{N} \sum_{i=1}^{N} \text{Tr}( \Sigma_{\mathcal{D}\mathcal{D}^{\prime}}^{*}(\tilde{\theta}_i) \Sigma_{\mathcal{D^\prime}\mathcal{D^\prime}}^{*-1}(\tilde{\theta}_i) \Sigma_{\mathcal{D}\mathcal{D}^{\prime}}^{*T}(\tilde{\theta}_i)).
\end{align}

The algorithm proceeds as follows. We sample $(\tilde{\theta}_i)_{i=1}^{N}$ from the prior. Initially we set $k=0$, $\mathcal{D}^{\prime} = \varnothing$ and $u_0=0$. We increment $k$ by one, and consider adding an inducing point. We then find the point $s_{k}^\prime$ that maximises $\tilde{\mathcal{U}}(\mathcal{D}^{\prime} \cup \{s\})$
\begin{align}
s_{k}^\prime := \underset{s \in \mathcal{S}}{\text{argmax }} \tilde{\mathcal{U}}(\mathcal{D}^{\prime} \cup \{s\})
\end{align}
using Bayesian optimisation \cite{Mockus13}. We compute the utility of having $k$ inducing points as \[u_{k} = \tilde{\mathcal{U}}(\mathcal{D}^{\prime} \cup \{s_k^\prime\}),\] we update $\mathcal{D}^\prime = \mathcal{D}^\prime \cup \{s_{k}\} $ and stop when \[ \frac{u_{k}-u_{k-1}}{u_{k}} < \alpha, \]
where $0< \alpha \ll1$ is a convergence threshold.

\begin{algorithm}[ht]
   \caption{Selection of inducing points}
   \label{al:inducing}
\begin{algorithmic}
   \STATE {\bfseries Inputs:}$0< \alpha \ll1$, $N$, $p_{\theta}$
   \STATE {\bfseries Output:} $u_{f}$, $\mathcal{D}^{\prime}$
   \STATE $k=0$, $u_0=0$, $\mathcal{D}^{\prime} = \varnothing$, $e=1$;
   \STATE  Sample $(\tilde{\theta}_i)_{i=1}^{N}$ from $p(\theta)$;
   \WHILE{$e> \alpha$}
    \STATE $k = k +1;$
    \STATE $s^\prime_k =  \underset{s \in \mathcal{S}}{\text{argmax }} \tilde{\mathcal{U}}(\mathcal{D}^{\prime} \cup \{s\});$
    \STATE $u_k=\tilde{\mathcal{U}}(\mathcal{D}^{\prime} \cup \{s^\prime_k\});$
    \STATE $\mathcal{D}^\prime = \mathcal{D}^\prime \cup \{s^\prime_k\};$
     \STATE $e = \frac{u_{k}-u_{k-1}}{u_{k}};$
   \ENDWHILE
\end{algorithmic}
\end{algorithm}
\textbf{Proposition}

\textit{(a) For any $\mathcal{D}$, $\alpha$, $N$ and $p_{\theta}$ Algorithm~\ref{al:inducing} stops in finite time and the sequence $(u_k)_{k \in \mathbb{N}}$ converges at least linearly with rate $1-\frac{1}{\#\mathcal{D}}$}.

\textit{(b) Moreover, the maximum utility $u_f(\alpha)$ returned by Algorithm~\ref{al:inducing} converges to the average total unconditional variance $w_{\infty} := \frac{1}{N} \sum_{i=1}^{N}  \text{Tr}(\Sigma_{\mathcal{D}\mathcal{D}}^{*}(\tilde{\theta}_i))$ as $\alpha$ goes to $0$}.

The idea behind the proof of this proposition is that the sequence of maximum utilities $u_k$ is positive, increasing\footnote{Intuitively, conditioning on a new point increases the reduction of variance from the unconditional variance.}, and upper-bounded by the total unconditional variance $w_{\infty}$\footnote{The variance cannot be reduced by more than the total unconditional variance.}. Hence, the sequence $u_k$ converges to a strictly positive limit, which implies that the stopping condition of the while loop will be met in finite time regardless of $\mathcal{D}$, $\alpha$, $N$ and $p_{\theta}$. Finally, we construct a sequence $w_k$ upper-bounded by the sequence $u_k$ and that converges linearly to the average total unconditional variance $w_{\infty}$ with rate $1-\frac{1}{\#\mathcal{D}}$. As the sequence $u_k$ converges and is itself upper-bounded by $w_{\infty}$, its limit is $w_{\infty}$ as well, and it converges at least as fast as $w_k$. \textit{(See appendix for the full proof)}

Our algorithm is particularly suitable to Poisson Point Processes as it prioritises sampling inducing points in parts of the domain where the data are denser. This corresponds to regions where the intensity function will be higher, thus where the local random counts of the underlying PPP will vary more\footnote{The variance of the Poisson distribution is its mean.} and subsequently where the posterior variance of the intensity is expected to be higher. Moreover, it leverages prior smoothness assumptions on the intensity function to limit the number of inducing points and to appropriately and sequentially improve coverage of the domain. 

Algorithm~\ref{al:inducing} is illustrated on a variety of real life and synthetic data sets in section \ref{sec:5}.

\section{INFERENCE}
We use a Squared Exponential kernel for $\gamma^{*}$ and \textit{Scaled Sigmoid Gaussian} priors for the kernel hyper-parameters; that is $\theta_i = \frac{\theta_{i\text{max}}}{1+\exp(-x_i)}$ where $x_i$ are i.i.d standard Normal. The problem-specific scales, $\theta_{i\text{max}}$, restrict the supports of those distributions using prior knowledge to avoid unlikely extreme values and to improve conditioning.

We use a Block Gibbs Sampler \cite{Gibbs84} to sample from the posterior. We sample the hyper-parameters using the Metropolis-Hastings \cite{Hastings70} algorithm taking as proposal distribution the prior of the variable of interest. We sample the log-intensities at the inducing points using Elliptical Slice Sampling \cite{Murray09b} with the pdf in Equation (\ref{eq:joint_int}).

\textbf{Prediction from training}

To predict the posterior mean at the data points we note from the law of total expectation that 
\begin{align}
\label{eq:tot_exp}
& \forall s_i \in \mathcal{D}, ~ \text{E}(\log{\lambda(s_i)}|\mathcal{D}) \nonumber \\ & = \text{E}\left(\text{E}\left(\log{\lambda(s_i)}|\{\log{\lambda^{*}(s^\prime_j)}\}_{j=1}^k, \mathcal{D}\right)|\mathcal{D}\right).
\end{align}
Also, we note from Equations (\ref{eq:likelihood_poisson}) and  (\ref{eq:tractable_joint_prior_condgp}) that the dependency of the posterior of $\log{\lambda(s_i)}$ conditional on $\{\log{\lambda^{*}(s^\prime_j)}\}_{j=1}^k$ is of the form \[\exp(\log{\lambda(s_i)}) \times \mathcal{N}(\log{\lambda(s_i)}|m(s_i), \gamma(s_i, s_i)),\] where we recall that $m(s_i)$ is the $i$-th element of the vector $M$ and $\gamma(s_i, s_i)$ is the $i$-th diagonal element of the matrix $\Sigma_{\mathcal{D}\mathcal{D}}$. Hence, the posterior distribution of $\log{\lambda(s_i)}$ conditional on  $\{\log{\lambda^{*}(s^\prime_j)}\}_{j=1}^k$ is Gaussian with mean 
\begin{align}
\label{eq:predictive_mean}
\text{E}\left(\log{\lambda(s_i)}|\{\log{\lambda^{*}(s^\prime_j)}\}_{j=1}^k, \mathcal{D} \right)=M[i]+\Sigma_{\mathcal{D}\mathcal{D}}[i,i]
\end{align}
and variance 
\begin{align}
\label{eq:predictive_variance}
\text{Var}\left(\log{\lambda(s_i)}|\{\log{\lambda^{*}(s^\prime_j)}\}_{j=1}^k, \mathcal{D} \right)= \Sigma_{\mathcal{D}\mathcal{D}}[i,i].
\end{align}
Finally, it follows from Equation (\ref{eq:tot_exp}) that $\text{E}(\log{\lambda(s_i)}|\mathcal{D})$ is obtained by averaging out $M[i]+\Sigma_{\mathcal{D}\mathcal{D}}[i,i]$ over MCMC samples after burn-in. 

Similarly, the law of total variance implies that
\begin{align}
& \text{Var}(\log{\lambda(s_i)}|\mathcal{D})  \nonumber \\
& = \text{E}\left(\text{Var}\left(\log{\lambda(s_i)}|\{\log{\lambda^{*}(s^\prime_j)}\}_{j=1}^k,\mathcal{D}\right)|\mathcal{D}\right)  \nonumber \\
& +\text{Var}\left( \text{E}\left(\log{\lambda(s_i)}|\{\log{\lambda^{*}(s^\prime_j)}\}_{j=1}^k, \mathcal{D}\right)|\mathcal{D}\right).
\end{align}

Hence, it follows from Equations (\ref{eq:predictive_mean}) and (\ref{eq:predictive_variance}) that the posterior variance at a data point $s_i$ is obtained by summing up the sample mean  of $\Sigma_{\mathcal{D}\mathcal{D}}[i,i]$ with the sample variance of $M[i]+\Sigma_{\mathcal{D}\mathcal{D}}[i,i]$, where sample mean and sample variance are taken over MCMC samples after burn-in.

\section{EXPERIMENTS}
\label{sec:5}
We selected four data sets to illustrate the performance of our model. We restricted ourselves to one synthetic data set for brevity. We chose the most challenging of the synthetic intensity functions of \cite{Murray09} and \cite{YWT11}, $\lambda(t) =2\exp(-\frac{t}{15}) + \exp(-(\frac{t-25}{10})^2)$, to thoroughly compare our model with competing methods. We also ran our model on a standard 1 dimensional real-life data set (the coal mine disasters dataset used in \cite{jarrett79}; 191 points) and a standard real-life 2 dimensional data (spatial location of bramble canes \cite{Diggle83}; 823 points). Finally we ran our model on a real-life data set large enough to cause problems to competing models. This data set consists of the UTC timestamps (expressed in hours in the day) of Twitter updates in English published in the \cite{Twitter14} on September 1st 2014 (188544 points).

\subsection{Inducing points selection}

Figure \ref{fig:convergence_plot} illustrates convergence of the selection of inducing points on the 4 data sets. We ran the algorithm 10 times with $N=20$, and plotted the average normalised utility $\frac{u_k}{u_{\infty}} \pm 1 \text{ std}$ as a function of the number of inducing points. Table \ref{table:hyper} contains the maximum hyper-parameters that were used for each data set. Table \ref{table:num_ind} contains the number of inducing points required to achieve some critical normalised utility values for each of the 4 data sets. We note that just 8 inducing points were required to achieve a 95\% utility for the Twitter data set (188544 points). In regards to the positions of sampled inducing points, we note from Figures \ref{fig:expo_gauss_plot} and \ref{fig:real_data_plot} that when the intensity function was bimodal, the first inducing point was sampled around the argument of the highest mode, and the second inducing point was sampled around the argument of the second highest mode. More generally, the algorithm sampled inducing points where the latent intensity function varies the most, as expected.

\begin{table}[t]
\caption{Maximum output (resp. input) scale $h_{\text{max}}$ (resp. $l_{\text{max}}$) used for each data set to select inducing points.}
\label{table:hyper}
\vskip 0.1in
\begin{center}
\begin{small}
\begin{sc}
    \begin{tabular}{  l | l | l | l | l}
    \hline
     & synthetic & coal mine & bramble & twitter \\ \hline
    $h_{\text{max}}$ & 10.0 & 10.0& 10.0 & 10.0  \\ \hline
    $l_{\text{max}}$ & 25.0 & 50.0 & 0.25 & 5.0  \\ \hline
    \end{tabular}
\end{sc}
\end{small}
\end{center}
\vskip -0.1in
\end{table}

\begin{figure}[!ht]
\centering
\includegraphics[width=0.5\textwidth]{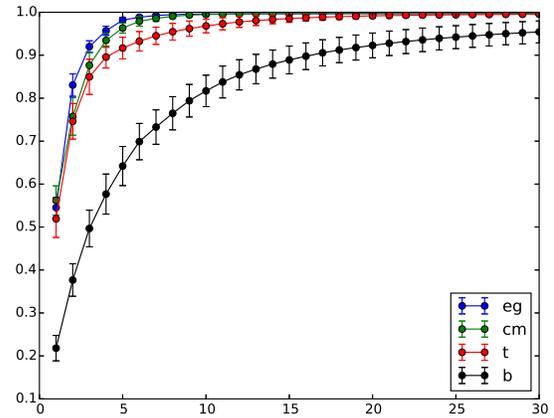}
\caption{Average normalised utility $\frac{u_k}{u_{\infty}}$ of choosing k inducing points using Algorithm~\ref{al:inducing} $\pm$ 1 standard deviation as a function of k on the synthetic data set (eg), the coal mine data set (cm), the Twitter data set (t) and the bramble canes data set (b). The average was taken over 10 runs.}
\label{fig:convergence_plot}
\end{figure}

\begin{table}
\caption{Number of inducing points produced by Algorithm \ref{al:inducing} required to achieve some critical normalised utility values on the 4 data sets.}
\label{table:num_ind}
\vskip 0.1in
\begin{center}
\begin{small}
\begin{sc}
\begin{tabular}{ l | l | l | l | l |}
    \cline{2-5}
  & \multicolumn{4}{c|} {$k$}  \\ \hline
  $\frac{u_k}{u_{\infty}}$ & synthetic & coal mine & bramble & twitter \\ \cline{1-5} \hline
    0.75& 2 & 2& 8 & 3  \\ \cline{1-5} \hline
    0.90 & 3 & 4 & 17 & 5 \\ \cline{1-5} \hline
    0.95 & 4 & 5 & 28 & 8 \\ \cline{1-5} \hline
\end{tabular}
\end{sc}
\end{small}
\end{center}
\vskip -0.1in
\end{table}

\subsection{Intensity function}
In each experiment we generated 5000 samples after burn-in (1000 samples). For each data set we used the set of inducing points that yielded a 95\% normalized utility. The exact numbers are detailed in Table \ref{table:num_ind}.

We ran a Monte Carlo simulation for the stochastic processes considered herein and found that the Legendre polynomial order $p=10$ was sufficient to yield a Quadrature estimate for the standard deviation of the integral less than 1\% away from the Monte Carlo estimate (using the trapezoidal rule), and a Quadrature estimate for the mean of the integral less than a standard error away from the Monte Carlo average. We took a more conservative stand and used $p=20$.

\textbf{Inference on synthetic data}

We generated a draw from a Poisson point process with the intensity function $\lambda(t) =2\exp(-\frac{t}{15}) + \exp(-(\frac{t-25}{10})^2)$ of \cite{Murray09} and \cite{YWT11}. The draw consisted of 41 points (blue sticks in Figure \ref{fig:expo_gauss_plot}). We compared our model to \cite{Murray09} (SGCP) and \cite{YWT11} (RMP). We ran the RMP model with the renewal parameter $\gamma$ set to 1 (RMP 1), which corresponds to an exponential renewal distribution or equivalently an inhomogeneous Poisson process. We also ran the RMP model with a uniform prior on $[1,5]$ over the renewal parameter $\gamma$ (RMP full). Figure \ref{fig:expo_gauss_plot} illustrates the posterior mean intensity function under each model. Finally we ran the Dirichlet Process Mixture of Beta model of \cite{Kottas06} (DPMB). As detailed in Table \ref{table:synth_comp}, our model outperformed that of \cite{Murray09}, \cite{YWT11} and \cite{Kottas06} in terms of accuracy and speed.

\begin{figure}[!ht]
\centering
\includegraphics[width=0.5\textwidth]{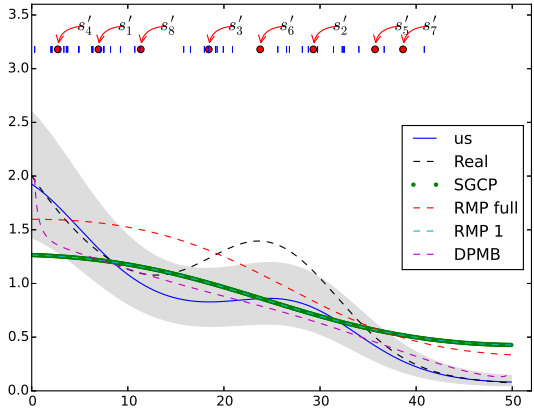}
\caption{Inference on a draw (blue sticks) from a Poisson point process with intensity $\lambda(t) =2\exp(-\frac{t}{15}) + \exp(-(\frac{t-25}{10})^2)$ (black line). The red dots are the inducing points generated by our algorithm, labelled in the order they were selected.  The solid blue line and the grey shaded area are the posterior mean $\pm$ 1 posterior standard deviation under our model. SGCP is the posterior mean under \cite{Murray09}. RMP full and RMP 1 are the posterior mean intensities under  \cite{YWT11} with $\gamma$ inferred and set to $1$ respectively. DPMB is the Dirichlet Process mixture of Beta \cite{Kottas06}}
\label{fig:expo_gauss_plot}
\end{figure}

\begin{table}
\caption{Some statistics on the MCMC runs of Figure \ref{fig:expo_gauss_plot}. RMSE and MAE denote the Root Mean Square Error and the Mean Absolute Error, expressed as a proportion of the average of the true intensity function over the domain. LP denotes the log mean predictive probability on 10 held out PPP draws from the true intensity $\pm$ 1 std. t(s) is the average time in seconds it took to generate 1000 samples $\pm$ 1 std and ESS denotes the average effective sample size \cite{Gelman13} per 1000 samples.}
\label{table:synth_comp}
\vskip 0.1in
   \begin{center}
   \begin{small}
   \begin{sc}
\resizebox{0.5\textwidth}{!}{
    \begin{tabular}{  l | l | l | l | l | l}
    \hline
     & MAE & RMSE & LP & t (s) & ESS\\ \hline
    SGCP & 0.31 & 0.37 & -45.07 $\pm$ 1.64  & 257.72 $\pm$ 16.29 & 6\\ \hline
    RMP 1 & 0.32 & 0.38 & -45.24 $\pm$ 1.41 & 110.19 $\pm$ 7.37 & 23  \\ \hline
    RMP full & 0.25 & 0.31 & -43.51 $\pm$ 2.15 & 139.64 $\pm$ 5.24 & 6 \\ \hline
    DPMB & 0.23 & 0.32 & -42.95 $\pm$ 3.58 & 23.27 $\pm$ 0.94  & \textbf{47} \\ \hline
    Us & \textbf{0.19} & \textbf{0.27} & \textbf{-42.84 $\pm$ 3.07} & \textbf{4.35 $\pm$  0.12} & 38\\ \hline
    \end{tabular}}
\end{sc}
\end{small}
\end{center}
\vskip -0.1in
\end{table}

\textbf{Inference on real-life data}

Figure \ref{fig:real_data_plot} shows the posterior mean intensity functions of the coal mine data set, the Twitter data set and the bramble canes data set under our model. 
\paragraph{Scalability:} We note that it took only 240s on average to generate 1000 samples on the Twitter data set (188544 points). As a comparison, this is the amount of time that would be required to generate as many samples on a data set that has 50 points (resp. 100 points) under the models of \cite{Murray09} (resp. \cite{YWT11}). More importantly, it was not possible to run either of those two competing models on the twitter data set. Doing so would require computing $17\times10^{10}$ covariance coefficients to evaluate a single auto-covariance matrix of the log-intensity at the data points, which a typical personal computer cannot handle.

\begin{figure}[!ht]
(a) \includegraphics[width=0.45\textwidth]{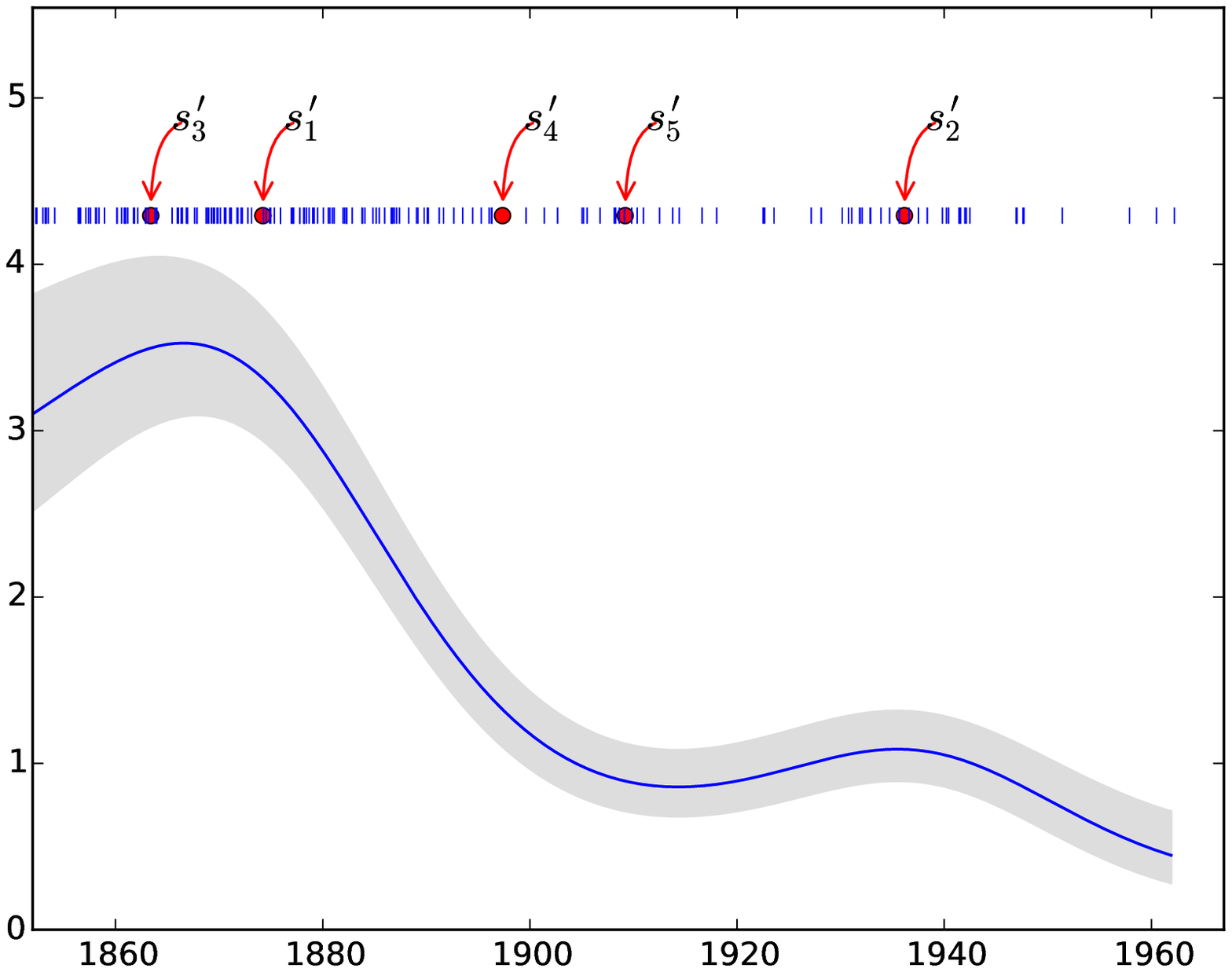}
(b) \includegraphics[width=0.45\textwidth]{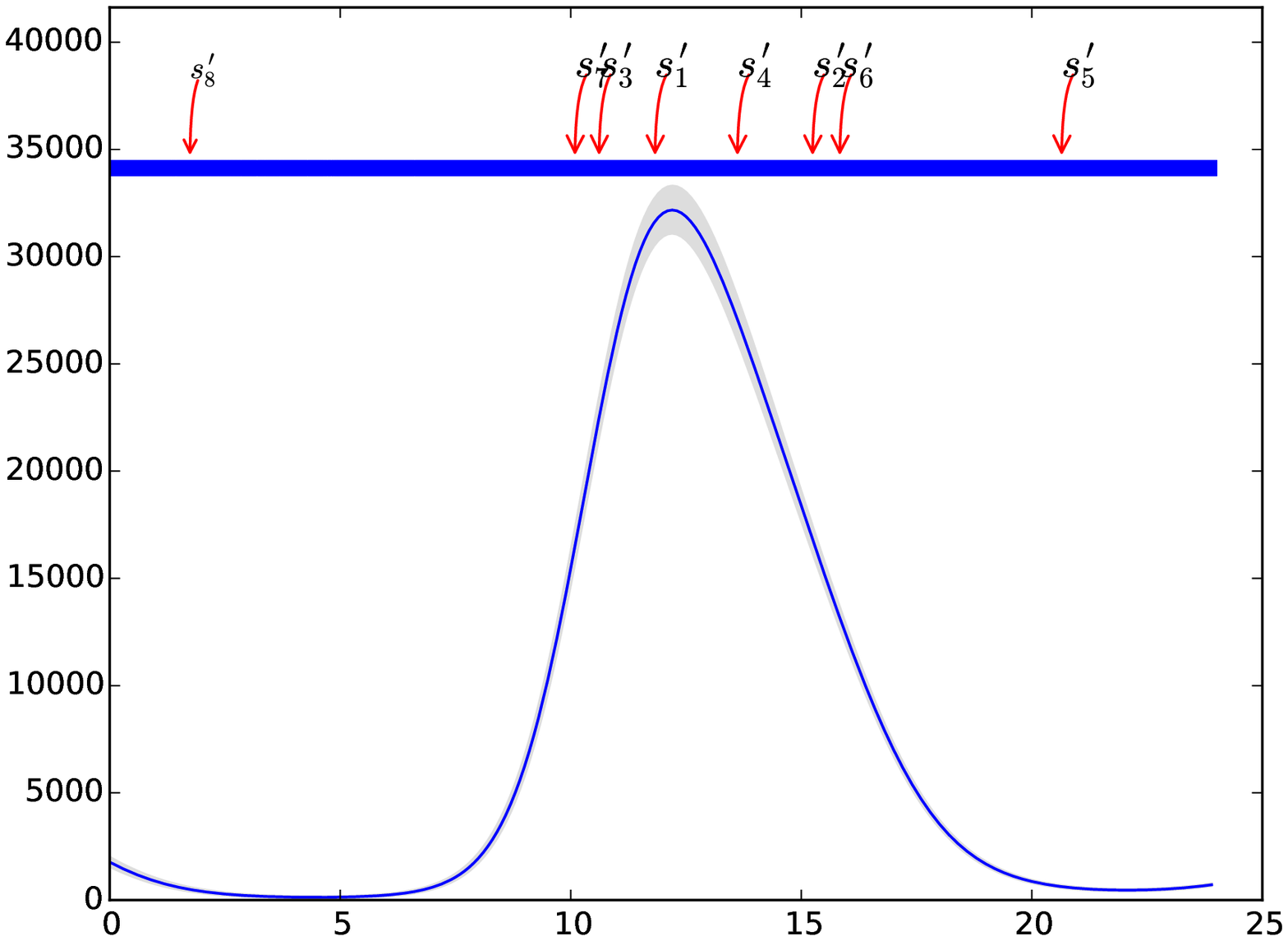}
(c) \includegraphics[width=0.5\textwidth]{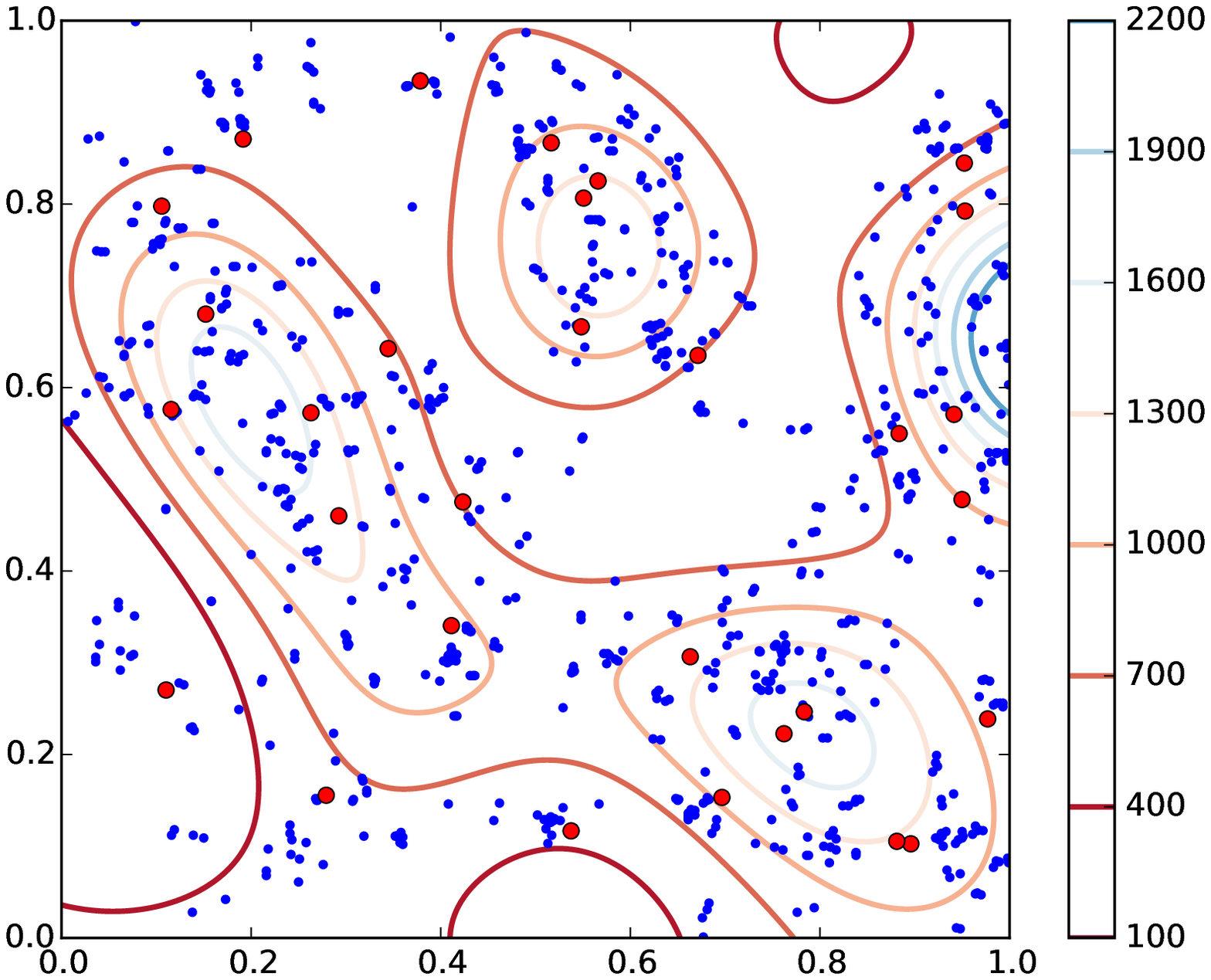}

\caption{Inference on the intensity functions of the coal mine data set (top), the twitter data set (middle), and the bramble canes data set (bottom). Blue dots are data points, red dots are inducing points (labelled in the upper panels in the order they were selected), the grey area is the 1 standard deviation confidence band.}
\label{fig:real_data_plot}
\end{figure}

\section{DISCUSSION}
\textbf{Scalability of the selection of inducing points}

The computational bottleneck of the selection of inducing points is in the evaluation of \[\text{Tr}( \Sigma_{\mathcal{D}\mathcal{D}^{\prime}}^{*}(\tilde{\theta}_i) \Sigma_{\mathcal{D^\prime}\mathcal{D^\prime}}^{*-1}(\tilde{\theta}_i) \Sigma_{\mathcal{D}\mathcal{D}^{\prime}}^{*T}(\tilde{\theta}_i)).\]Hence, the complexity and the memory requirement of the selection of inducing points are both linear in the number of data points $n:= \#\mathcal{D}$.

The number of inducing points generated by our algorithm does not increase with the size of the data, but rather as a function of the size of the domain and the resolution implied by the prior over the hyper-parameters. 

\textbf{Comparison with competing models}

We note that the computational bottleneck of our MCMC inference is in the evaluation of \[\text{Tr}(\Sigma_{\mathcal{D}\mathcal{D}}) = \text{Tr}(\Sigma^{*}_{\mathcal{D}\mathcal{D}}) - \text{Tr}(\Sigma^{*}_{\mathcal{D}\mathcal{D}^\prime}\Sigma^{*-1}_{\mathcal{D}^\prime\mathcal{D}^\prime}\Sigma^{*T}_{\mathcal{D}\mathcal{D}^\prime}).\]
Hence, inferring the intensity function under our model scales computationally in $\mathcal{O}(nk^2)$ and has a memory requirement  $\mathcal{O}(nk)$, where the number of inducing points $k$ is negligible. This is considerably better than alternative methods using Gaussian Processes \cite{Murray09, YWT11} whose complexities are cubic in the number of data points and whose memory requirement is squared in the number of data points. Moreover, the superior accuracy of our model compared to \cite{Murray09} and \cite{YWT11} is due to our use of the exponential transformation rather than the scaled sigmoid one. In effect, unlike the inverse scaled sigmoid function that tends to amplify variations, the logarithm tends to smooth out variations. Hence, when the true intensity is uneven, the log-intensity is more likely to resemble a draw from a stationary GP than the inverse scaled sigmoid of the true intensity function, and subsequently a stationary GP prior in the inverse domain is more suitable to the exponential transformation than to the scaled sigmoid transformation.

Our model is also more suitable than that of \cite{Cunni08} when confidence bounds are needed for the intensity function, or when the input space is of dimension higher than 1. The model is a useful alternative to that of \cite{Kottas06}, whose complexity is also linear. In effect, Gaussian Processes (GP) are more flexible than a Dirichlet Process (DP) mixture of Beta distributions. This is the result of the large number of known covariance kernels available in the literature and the state-of-the-art understanding of how well a given kernel can approximate an arbitrary function \cite{Micchelli06, Pillai07}. Moreover, unlike a Dirichlet Process mixture of Beta distributions, Gaussian Processes allow directly expressing practical prior features such as smoothness, amplitude, length scale(s) (memory), and periodicity. 

As our model relies on the Gauss-Legendre quadrature, we would not recommend it for applications with a large input space dimension. However, most interesting point process applications involve modelling temporal, spatial or spatio-temporal events, for which our model scales considerably better with the data size than competing approaches. In effect, the models proposed by \cite{Kottas06, Cunni08, Cunni08b, YWT11} are all specific to unidimensional input data, whereas the model introduced by \cite{Kottas07} is specific to spatial data. As for the model of \cite{Murray09}, it scales very poorly with the input space dimension for its complexity is cubic in the sum of the number of data points and the number of latent thinning points, and the number of thinning points grows exponentially with the input space dimension\footnote{The expected number of thinning points grows proportionally with the volume of the domain, which is exponential in the dimension of the input space when the domain is a hypercube with a given edge length.}.

\textbf{Extension of our model}

Although the covariance kernel $\gamma^{*}$ was assumed stationary, no result in this paper relied on that assumption. We solely needed to evaluate covariance matrices under $\gamma^{*}$. Hence, the proposed model and algorithm can also be used to account for known non-stationarities. More generally, the model presented in this paper can serve as foundation to make inference on the stochastic dependency between multiple point processes when the intensities are assumed to be driven by known exogenous factors, hidden common factor, and latent idiosyncratic factors.

\section{SUMMARY}
In this paper we propose a novel exact non-parametric model to make inference on Poisson Point Processes using Gaussian Processes. We derive a robust MCMC scheme to sample from the posterior intensity function. Our model outperforms competing benchmarks in terms of speed and accuracy as well as in the decorrelation of MCMC samples. A critical advantage of our approach is that it has a numerical complexity and a memory requirement \emph{linear} in the data size $n$ ($\mathcal{O}(nk^2)$, and $\mathcal{O}(nk)$ respectively, with $k \ll n$). Competing models using Gaussian Processes have a cubic numerical complexity and squared memory requirement. We show that our model readily handles data sizes not yet considered in the literature.

\section*{Acknowledgments}
Yves-Laurent Kom Samo is supported by the Oxford-Man Institute of Quantitative Finance.

\appendix
\section*{Appendix}
\renewcommand{\thesubsection}{\Alph{subsection}}

\subsection{There exists a Cox process with  an a.s. $\mathcal{C}^\infty$ intensity coinciding with any finite dimensional prior.}
In this section we prove the proposition below.
\begin{proposition}
Let $\mathbb{Q}$ be an $(n+1)$ dimensional continuous probability distribution whose density has support $\bigotimes_{i=1}^{n+1} ~]0, +\infty[$, and let $x_1, \dots, x_n$ be $n$ points on a compact domain $\mathcal{S} \subset \mathbb{R}^d$. There exists an almost surely non-negative and $\mathcal{C}^{\infty}$ stochastic process $\lambda$ on $\mathcal{S}$ such that
\[\big(\lambda(x_1),..., \lambda(x_n), \int_{\mathcal{S}} \lambda(x) dx\big) \sim \mathbb{Q}.\]
\end{proposition}
\begin{proof}
Let \[(y_1, \dots, y_n, I)  \sim \mathbb{Q}\] and \[(y_1(\omega), \dots, y_n(\omega), I(\omega))\] a random draw. Let us denote $x^j, j \leq d$ the $j$-th coordinate of $x \in \mathbb{R}^d$. We consider the family of functions parametrized by $\alpha \in \mathbb{R}$:
\begin{align}
f(\omega, x, \alpha) =& \exp\bigg(\alpha \sum_{ j=1}^{d} \prod_{l=1}^{n} (x^j-x_l^j)^2\bigg) \\ \nonumber
& \times \sum_{l=1}^n y_l(\omega) \frac{1}{d}\sum_{j=1}^d \prod_{k \neq l} \bigg(\frac{x^j-x_k^j}{x^j_l-x_k^j}\bigg)^2.
\end{align}
We note that $\forall \alpha, x_i,~ f(\omega, x_i, \alpha) = y_i(\omega)$. Let us define the polynomial \[P(x)=\sum_{l=1}^n y_l(\omega) \frac{1}{d}\sum_{j=1}^d \prod_{k \neq l} \bigg(\frac{x^j-x_k^j}{x^j_l-x_k^j}\bigg)^2.\] As $P$ is continuous, it is bounded on the compact $\mathcal{S}$, and reaches its bounds. Thus we have 
\[ \exists ~m_p, M_p \geq 0, ~ \text{s.t.}~ \forall x \in \mathcal{S}, ~  0 \leq m_p \leq P(x) \leq M_p.\]
Similarly, if we define \[R(\alpha, x)=  \exp\bigg(\alpha \sum_{ j=1}^{d} \prod_{l=1}^{n} (x^j-x_l^j)^2\bigg) = R(1, x)^\alpha,\] it follows that
\[ \exists ~m_q, M_q > 1, ~ \text{s.t.}~ \forall x \in \mathcal{S}, ~ 1 <  m_q \leq R(1, x) \leq M_q.\]
Hence, 
\begin{equation}
\label{eq:int_bound}
 m_p m_q^\alpha \mu(\mathcal{S}) \leq  \int_{\mathcal{S}} f(\omega, x, \alpha) dx \leq M_p M_q^\alpha \mu(\mathcal{S}).
 \end{equation}
Moreover, we note that $\alpha \to \int_{\mathcal{S}} f(\omega, x, \alpha) dx$ is continuous on $\mathbb{R}$ as its restriction to any bounded interval is continuous (by dominated convergence theorem). Furthermore, given that $m_q, M_q >1$, it follows from Equation (\ref{eq:int_bound}) that \[\underset{ \alpha \to +\infty}{\text{lim }} \int_{\mathcal{S}} f(\omega, x, \alpha) dx = +\infty\] and \[\underset{ \alpha \to -\infty}{\text{lim }} \int_{\mathcal{S}} f(\omega, x, \alpha) dx = 0.\]
Hence, by intermediate value theorem, 
\[\forall ~ I(\omega) > 0, ~ \exists \alpha^* (\omega) ~\text{s.t.}~ I(\omega) = \int_{\mathcal{S}} f(\omega, x, \alpha^*(\omega)) dx.\]
Finally, let us define the stochastic process $\lambda$ on $\mathcal{S}$ as \[\omega \to \lambda(\omega, x) := f(\omega, x, \alpha^*(\omega)).\] 
To summarise,
\[\forall ~ x_i, \lambda(\omega, x_i) :=  f(\omega, x_i, \alpha^*(\omega)) = y_i(\omega),\]
\[I(\omega) = \int_{\mathcal{S}} \lambda(\omega, x) dx,\] 
and 
\[(y_1, \dots, y_n, I) \sim \mathbb{Q}:\]
this implies $\big(\lambda(x_1),..., \lambda(x_n), \int_{\mathcal{S}} \lambda(x) dx\big) \sim \mathbb{Q}.$
Finally,
\[\forall ~ x \in \mathcal{S}, ~ \lambda(\omega, x) \geq 0, \text{ and } \forall ~ \omega, ~ x \to \lambda(\omega, x) \text{ is } \mathcal{C}^{\infty},\]
which concludes our proof.
\end{proof}

\subsection{Proof of convergence of Algorithm 1}
The idea behind the proof is to show that the sequence of maximum utility \[u_k = \underset{s \in \mathcal{S}}{\text{max }} \tilde{\mathcal{U}}(\{s^\prime_1, ..., s^\prime_{k-1}\} \cup \{s\})\]
is positive, increasing and upper-bounded and thus converges to a strictly positive limit. This would then imply that \[ \frac{u_{k+1}-u_{k}}{u_k}  \underset{k \to \infty}{\longrightarrow} 0\] and subsequently that \[ \forall ~ 0 < \alpha < 1,  \exists ~ k_{\text{lim}} \in \mathbb{N} \text{ s.t. } \forall ~ k > k_{\text{lim}}, \frac{u_{k+1}-u_{k}}{u_k} < \alpha \] or in other words Algorithm 1 always stops in finite time.

To show that $\forall k>0, ~ u_k>0$, we note that $\Sigma_{\mathcal{D^\prime}\mathcal{D^\prime}}^{*}(\tilde{\theta}_i)$ is a covariance matrix and as such it is positive definite. It follows that $\Sigma_{\mathcal{D^\prime}\mathcal{D^\prime}}^{*-1}(\tilde{\theta}_i)$ is also positive definite. We further note that the j-th diagonal term of $\Sigma_{\mathcal{D}\mathcal{D}^{\prime}}^{*}(\tilde{\theta}_i) \Sigma_{\mathcal{D^\prime}\mathcal{D^\prime}}^{*-1}(\tilde{\theta}_i) \Sigma_{\mathcal{D}\mathcal{D}^{\prime}}^{*T}(\tilde{\theta}_i)$ can be written as $x_j^T\Sigma_{\mathcal{D^\prime}\mathcal{D^\prime}}^{*-1}(\tilde{\theta}_i) x_j$ where $x_j$ is the j-th column of $\Sigma_{\mathcal{D}\mathcal{D}^{\prime}}^{*T}(\tilde{\theta}_i)$. Hence, by virtue of the positive definitiveness of $\Sigma_{\mathcal{D^\prime}\mathcal{D^\prime}}^{*-1}(\tilde{\theta}_i)$, the diagonal terms of 
$\Sigma_{\mathcal{D}\mathcal{D}^{\prime}}^{*}(\tilde{\theta}_i) \Sigma_{\mathcal{D^\prime}\mathcal{D^\prime}}^{*-1}(\tilde{\theta}_i) \Sigma_{\mathcal{D}\mathcal{D}^{\prime}}^{*T}(\tilde{\theta}_i)$ are all positive, which proves that the utility function $\tilde{\mathcal{U}}$ is positive, and subsequently that $\forall k>0, u_k >0$.

To show that $(u_k)_{k \in \mathbb{N}^{*}}$ is upper-bounded, we note that the matrix \[C_{i\mathcal{D}^{\prime}} = \Sigma_{\mathcal{D}\mathcal{D}}^{*}(\tilde{\theta}_i) - \Sigma_{\mathcal{D}\mathcal{D}^{\prime}}^{*}(\tilde{\theta}_i) \Sigma_{\mathcal{D^\prime}\mathcal{D^\prime}}^{*-1}(\tilde{\theta}_i) \Sigma_{\mathcal{D}\mathcal{D}^{\prime}}^{*T}(\tilde{\theta}_i)\] where the notation is as per the rest of the paper, is an auto-covariance matrix, and as such has positive diagonal elements. Hence, \[\text{Tr}(\Sigma_{\mathcal{D}\mathcal{D}}^{*}(\tilde{\theta}_i)) \geq \text{Tr}(\Sigma_{\mathcal{D}\mathcal{D}^{\prime}}^{*}(\tilde{\theta}_i) \Sigma_{\mathcal{D^\prime}\mathcal{D^\prime}}^{*-1}(\tilde{\theta}_i) \Sigma_{\mathcal{D}\mathcal{D}^{\prime}}^{*T}(\tilde{\theta}_i))\] and finally \[\forall ~ k \in \mathbb{N}^{*}, u_k \leq \frac{1}{N} \sum_{i=1}^{N}  \text{Tr}(\Sigma_{\mathcal{D}\mathcal{D}}^{*}(\tilde{\theta}_i)).\]

Moreover, we note that showing that $(u_k)_{k \in \mathbb{N}^{*}}$ is increasing is equivalent to showing that  $(v_k)_{k \in \mathbb{N}^{*}}$ with \[v_k = \underset{s \in \mathcal{S}}{\text{min }}\frac{1}{N} \sum_{i=1}^{N}  \text{Tr}(C_{i\{s^\prime_1, ..., s^\prime_{k-1}\} \cup \{s\}})\] is decreasing. We recall that $C_{i\{s^\prime_1, ..., s^\prime_{k-1}\} \cup \{s\}}$ is the covariance matrix of the values of the stationary Gaussian Process of our model at the data points, conditioned on its values at $\{s^\prime_1, ..., s^\prime_{k-1}\} \cup \{s\}$. 

It follows from the law of iterated expectations that $C_{i\{s^\prime_1, ..., s^\prime_{k-1}\} \cup \{s\}}$ could also be seen as the covariance matrix of the values of a conditional Gaussian Process at the data points, \footnote{The conditional GP is defined as the stationary Gaussian Process in our model is conditioned on its values at the points $\{s^\prime_1, ..., s^\prime_{k-1}\}$} conditioned on its value at $s$. Hence,

\begin{align}
&C_{i\{s^\prime_1, ..., s^\prime_{k-1}\} \cup \{s\}} = \nonumber \\
&C_{i\{s^\prime_1, ..., s^\prime_{k-1}\}} - \frac{1}{\hat{\Sigma}_{ss}(\tilde{\theta}_i)}\hat{\Sigma}_{\mathcal{D}\{s\}}(\tilde{\theta}_i)\hat{\Sigma}_{\mathcal{D}\{s\}}^T(\tilde{\theta}_i) \nonumber
\end{align}

where $\hat{\Sigma}_{XY}$ denotes the covariance matrix between the values of the conditional GP at points in X and at points in Y. In particular, $\hat{\Sigma}_{ss}(\tilde{\theta}_i)$ is a positive scalar. What's more the diagonal elements of $\hat{\Sigma}_{\mathcal{D}\{s\}}(\tilde{\theta}_i)\hat{\Sigma}_{\mathcal{D}\{s\}}^T(\tilde{\theta}_i)$ are all non-negative. Hence, \[\forall s \in \mathcal{S}, \text{Tr}(C_{i\{s^\prime_1, ..., s^\prime_{k-1}\} \cup \{s\}}) \leq \text{Tr}(C_{i\{s^\prime_1, ..., s^\prime_{k-1}\}}) \] and averaging over the set of hyper-parameters $\theta_i$ and taking the min we get \[\forall ~ k \geq 2, v_k \leq v_{k-1}\] which concludes the proof.

\subsection{Proof of the rate of convergence of Algorithm 1 and that $u_{f}$ in Algorithm 1 converges to $\frac{1}{N} \sum_{i=1}^{N}  \text{Tr}(\Sigma_{\mathcal{D}\mathcal{D}}^{*}(\tilde{\theta}_i))$ as $\alpha$ goes to $0$}

The key idea of this proof is to note as previously shown that no set of inducing points has a utility greater than $w_{\infty} :=\frac{1}{N} \sum_{i=1}^{N}  \text{Tr}(\Sigma_{\mathcal{D}\mathcal{D}}^{*}(\tilde{\theta}_i))$, but that any set of inducing points that includes $\mathcal{D}$ has a utility equal to $ w_{\infty}$.

Let $\{s_1^\prime, ..., s_k^\prime\}$ be points selected after $k$ iterations of Algorithm 1, and let us denote by $\{u_1, ..., u_k\}$ the maximum utilities after the corresponding iterations as usual. Let us denote by \[\tilde{s}_k = \underset{s \in \mathcal{D}}{\text{argmax }} \tilde{\mathcal{U}}(\{s^\prime_1, ..., s^\prime_{k-1}\} \cup \{s\}) \] the best candidate \textit{in the data set} to be the k-th inducing point after $k-1$ iterations of our algorithm. As previously mentioned, $\{s^\prime_1, ..., s^\prime_{k-1}\} \cup \mathcal{D}$ is a set of inducing points with perfect utility. Therefore, if we select the data points as inducing points after $\{s^\prime_1, ..., s^\prime_{k-1}\}$, their contribution to the overall utility will be $w_{\infty}-u_{k-1}$. If we further constrain our choice of $\mathcal{D}$ as additional inducing points to start with $\tilde{s}_k$ then the incremental utility of choosing $\tilde{s}_k$ will be at least $\frac{w_{\infty}-u_{k-1}}{n}$, where $n$ is the data size as usual. This is because $\tilde{s}_k$ is the best choice for the k-th inducing point in $\mathcal{D}$ after having picked $\{s^\prime_1, ..., s^\prime_{k-1}\}$ and because the incremental utility of choosing an inducing point is higher earlier (when little is known about the GP) than later (when more is known about the GP). What's more, by definition, the incremental utility of choosing $s_k^\prime$ after $\{s^\prime_1, ..., s^\prime_{k-1}\}$ is higher than that of choosing $\tilde{s}_k$ after $\{s^\prime_1, ..., s^\prime_{k-1}\}$. Hence, \[ u_k - u_{k-1} \geq \frac{w_{\infty}-u_{k-1}}{n}.\]
Let us denote by $w_k$ the sequence satisfying \[w_0 = u_0, \forall ~ k \in \mathbb{N}^{*} w_k - w_{k-1} = \frac{w_{\infty}-w_{k-1}}{n}.\] It can be shown (by induction on k) that \[\forall ~ k \in \mathbb{N}^{*}  w_k \leq u_k.\]

Moreover, we note that \[w_k - w_{\infty} = (1-\frac{1}{n})(w_{k-1} - w_{\infty}).\] Hence \[w_k = w_{\infty} + (1-\frac{1}{n})^k (w_0 -w_{\infty}),\] which proves that the sequence $w_{k}$ converges linearly to $w_{\infty}$ with rate $1-\frac{1}{n}$.

On one hand, we have shown that the sequence $u_k$ converges and is upper-bounded by $w_{\infty}$, hence its limit is smaller than $w_{\infty}$: \[u_{\infty} := \underset{k \to \infty}{\text{lim }} u_k \leq w_{\infty}.\]

On the other hand, we have shown that $\forall ~ k \in \mathbb{N}^{*} ~ w_k \leq u_k$ which implies \[w_{\infty} \leq u_{\infty}.\] Hence, \[w_{\infty} = u_{\infty} = \frac{1}{N} \sum_{i=1}^{N}  \text{Tr}(\Sigma_{\mathcal{D}\mathcal{D}}^{*}(\tilde{\theta}_i)).\]

As $w_k$ is upper-bounded by $u_k$ and both sequences converge to the same limit, $u_k$, and subsequently Algorithm 1, converge at least as fast as $w_k$.  

In regards to the second statement of our proposition, we have that \[ \underset{\alpha \to 0}{\text{lim }} u_{f}(\alpha) = \underset{k \to \infty}{\text{lim }} u_k = \frac{1}{N} \sum_{i=1}^{N}  \text{Tr}(\Sigma_{\mathcal{D}\mathcal{D}}^{*}(\tilde{\theta}_i)).\]

\newpage
\bibliography{icml_2015_arxiv}

\begin{thebibliography}{24}
\providecommand{\natexlab}[1]{#1}
\providecommand{\url}[1]{\texttt{#1}}
\expandafter\ifx\csname urlstyle\endcsname\relax
  \providecommand{\doi}[1]{doi: #1}\else
  \providecommand{\doi}{doi: \begingroup \urlstyle{rm}\Url}\fi


\bibitem[Adams et~al.(2009)Adams, Murray, and MacKay]{Murray09}
Adams, R.P., Murray, I., and MacKay, D.J.C.
\newblock Tractable nonparametric bayesian inference in Poisson processes with
  gaussian process intensities.
\newblock pp.\  9--16, 2009.

\bibitem[Basu \& Dassios(2002)]{basu02}
Basu, S.and Dassios, A. (2002)
\newblock A Cox process with log-normal intensity.
\newblock \emph{Insurance: mathematics and economics}, 31 (2). pp. 297-302. ISSN 0167-6687

\bibitem[Cox(1955)Cox]{Cox55}
Cox, D.R.
\newblock Some Statistical Methods Connected with Series of Events.
\newblock \emph{Journal of the Royal Statistical Society}, 17:\penalty2 129--164, 1955.

\bibitem[Cox \& Isham(1980)Cox, Isham]{Cox80}
Cox, D.R., Isham, V. (eds.).
\newblock \emph{Point Processes}.
\newblock Chapman Hall/CRC, 1980.

\bibitem[Cunningham et~al.(2008a)Cunningham, Shenoy, and Sahani]{Cunni08}
Cunningham, J.P., Shenoy, K.V., and Sahani, M.
\newblock Fast Gaussian Process Methods for Point Process Intensity Estimation.
\newblock Appearing in Proceedings of the 25 th International Conference on Machine Learning, Helsinki, Finland, 2008.

\bibitem[Cunningham et~al.(2008b)Cunningham, Yu, Shenoy, and Sahani]{Cunni08b}
Cunningham, J.P., Yu, B., Shenoy, K.V., and Sahani, M.
\newblock Inferring neural firing rates from spike trains using Gaussian Processes.
\newblock Advances in Neural Information Processing Systems 20 (pp. 329–336).

\bibitem[Daley \& Vere-Jones(2008)Daley and Vere-Jones]{Daley08}
Daley, D.J. and Vere-Jones, D.
\newblock \emph{An Introduction to the Theory of Point Processes}.
\newblock Springer-Verlag, 2008.

\bibitem[Diggle(1983)Diggle]{Diggle83}
Diggle, P.J.
\newblock \emph{Statistical Analysis of Spatial Point Patterns}.
\newblock Academic Press.

\bibitem[Diggle(1985)Diggle]{Diggle85}
Diggle, P.J.
\newblock \emph{A kernel method for smoothing point process data}.
\newblock Applied Statistics, 34:\penalty0 138–-147, 1985.

\bibitem[Gelman et~al.(2013)Gelman, Carlin, Stern, Dunson, Vehtari, and
  Rubin]{Gelman13}
Gelman, A., Carlin, J.B., Stern, H.S., Dunson, D.B., Vehtari, A., and Rubin,
  D.B. (eds.).
\newblock \emph{Bayesian Data Analysis Thrid Edition}.
\newblock CRC Press, 2013.

\bibitem[Geman \& Geman(1984)Geman and Geman]{Gibbs84}
Geman, S. and Geman, D.
\newblock Stochastic relaxation, Gibbs distributions, and the Bayesian
  restoration of images.
\newblock \emph{IEEE Transactions on Pattern Analysis and Machine
  Intelligence}, 6:\penalty0 721--741, 1984.

\bibitem[Gregory \& Loredo (1992)]{Gregory92}
Gregory, P. C., Loredo, T. J.
\newblock A new method for the detection of a periodic signal of unknown shape and period.
\newblock \emph{The Astrophysical Journal}, The Astrophysical Journal, 398, 146–168, 1992.

\bibitem[Gunter et al.(2014)Gunter, Lloyd, Osborne, Roberts]{gunter}
Gunter, T., Lloyd, C., Osborne, M.A., Roberts, S.J.
\newblock Efficient Bayesian Nonparametric Modelling of Structured Point Processes.
\newblock \emph{Uncertainty in Artificial Intelligence (UAI)}, 2014.

\bibitem[Hastings(1970)]{Hastings70}
Hastings, W.K.
\newblock Monte Carlo sampling methods using markov chains and their
  applications.
\newblock \emph{Biometrika}, 24:\penalty0 97--109, 1970.

\bibitem[Heikkinen \& Arjas(1999)]{heik99}
Heikkinen, J., Arjas, E.
\newblock Modeling a Poisson forest in variable elevations: a nonparametric Bayesian approach.
\newblock \emph{Biometrics}, 55, 738–745, 1999.

\bibitem[Hildebrand(2003)Hildebrand]{Hilde56}
Hildebrand, F. B, 
\newblock \emph{Introduction to Numerical Analysis: Second Edition}, Chap. 8.
\newblock Dover Publications, Inc., 2003.

\bibitem[Jarrett(1979)Jarrett]{jarrett79}
Jarrett, R.G.
\newblock A note on the intervals between coal-mining disasters.
\newblock \emph{Biometrika}, 66, 191-193.

\bibitem[Kingman(1992)Kingman]{Kingman92}
Kingman\newblock \emph{Poisson Processes}.
\newblock Oxford Science Publications, 1992.
  
\bibitem[Kottas(2006)Kottas]{Kottas06}
Kottas, A.
\newblock Dirichlet process mixtures of beta distributions, with applications to density and intensity estimation.
\newblock \emph{In Proceedings of the Workshop on Learning with Nonparametric Bayesian Methods, 23rd ICML, Pittsburgh, PA, 2006}.

\bibitem[Kottas \& Sanso(2007)Kottas and Sanso]{Kottas07}
Kottas, A., and Sanso, B.
\newblock Bayesian mixture modeling for spatial Poisson process intensities, with applications to extreme value analysis. Journal of Statistical Planning and Inference.
\newblock \emph{Journal of Statistical Planning and Inference}, 137, 3151–-3163, 2007.

\bibitem[Micchelli et~al.(2006)Micchelli, Xu, and Zhang]{Micchelli06}
Micchelli, C.A., Xu, Y., Zhang, H.
\newblock Universal Kernels.
\newblock \emph{Journal of Machine Learning Research}, 7 (2006) 2651-2667.

\bibitem[Metropolis et~al.(1953)Metropolis, Rosenbluth, Rosenbluth, Teller, and
  Teller]{Metropolis53}
Metropolis, N., Rosenbluth, A.W, Rosenbluth, M.N., Teller, A.H., and Teller, E.
\newblock Equations of state calculations by fast computing machines.
\newblock \emph{Journal of Chemical Physics}, 24:\penalty0 1087--1092, 1953.

\bibitem[Mockus(2013)Mockus]{Mockus13}
Mockus, J.
\newblock Bayesian approach to global optimization: theory and applications.
\newblock Kluwer Academic, 2013.

\bibitem[Moeller et~al.(1998)Moeller, Syversveen, and Waagepetersen]{Moller98}
Moeller, J., Syversveen, A., and Waagepetersen, R.
\newblock Log-gaussian cox processes.
\newblock \emph{Scandinavian Journal of Statistics}, 1998.

\bibitem[Murray et~al.(2010)Murray, Adams, and MacKay]{Murray09b}
Murray, I., Adams, R.P., and MacKay, D.J.C.
\newblock Elliptical slice sampling.
\newblock pp.\  9--16. Appearing in Proceedings of the 13th International
  Conference on Artificial Intelligence and Statistics (AISTATS), 2010.

\bibitem[Pillai et~al.(2007)Pillai, Wu, Liang, Mukherjee, and
  Wolpert]{Pillai07}
Pillai, N.S., Wu, Q., Liang, F., Mukherjee, S., and Wolpert,
  R.L.
\newblock Characterizing the function space for Bayesian kernel models.
\newblock \emph{Journal of Machine Learning Research}, 8:\penalty0 1769--1797,
  2007.

\bibitem[Quinonero \& Rasmussen(2005)Joaquin Quinonero-Candela, Carl Edward Rasmussen]{FTCI}
Quinonero-Candela, J. and Rasmussen C.E.
\newblock A Unifying View of Sparse Approximate Gaussian Process Regression.
\newblock \emph{Journal of Machine Learning Research}, 6 (2005) 1939–1959.

\bibitem[Rao \& Teh(2011)Rao and Teh]{YWT11}
Rao, V.~A. and Teh, Y.~W.
\newblock Gaussian process modulated renewal processes.
\newblock \emph{Neural Information Processing Systems (NIPS)}, 2011.

\bibitem[Rassmussen \& Williams(2006)Rassmussen and Williams]{Rassmussen06}
Rassmussen, Carl~E. and Williams, Christopher~K.I. (eds.).
\newblock \emph{Gaussian Processes for Machine Learning}.
\newblock The MIT Press, 2006.

\bibitem[Rathbum \& Cressie(1994)Rathbum and Cressie]{Rathbum94}
Rathbum, S.L. and Cressie, N.A.C.
\newblock Asymptotic properties of estimators for the parameters of spatial
  inhomogeneous Poisson point processes.
\newblock \emph{Advances in Applied Probability}, 26:\penalty0 122--154, 1994.

\bibitem[Twitter Sample Stream(2014)]{Twitter14}
Twitter Inc.
\newblock Twitter sample stream API. https://dev.twitter.com/streaming/reference/get/
statuses/sample
\end{thebibliography}
\bibliographystyle{aistats2014}
\end{document}